\RequirePackage{fix-cm}
\documentclass[smallextended]{svjour3}       
\smartqed  
\usepackage{graphicx}

\RequirePackage[round]{natbib}
\usepackage{lineno}
\usepackage[unicode]{hyperref}
\modulolinenumbers[5]
\usepackage{xcolor}

\usepackage{amssymb,amsmath}
\usepackage{array,multirow,graphicx,setspace,enumerate,ulem}
\usepackage{rotating}

\usepackage[textsize=small,textwidth=3cm,shadow]{todonotes}
\newcommand{\accom}[1]{\todo[color=green!50!white]{\'{A}: #1}} 

\DeclareMathOperator{\size}{size}

\def\mt{t\kern-0.035cm\char39\kern-0.03cm}
\def\ml{l\kern-0.035cm\char39\kern-0.03cm}
\def\md{d\kern-0.035cm\char39\kern-0.03cm}
\def\rho{\phi}
\newtheorem{thm}{Theorem}
\newtheorem{df}{Definition}

\newtheorem{obs}{Observation}
\newlength{\fs}
\usepackage{tikz}
\usetikzlibrary{arrows,decorations.markings,patterns,positioning,shapes,decorations.pathreplacing,calc}
\tikzset{draw half paths/.style 2 args={%
  decoration={show path construction,
    lineto code={
      \draw [#1] (\tikzinputsegmentfirst) -- 
         ($(\tikzinputsegmentfirst)!0.5!(\tikzinputsegmentlast)$);
      \draw [#2] ($(\tikzinputsegmentfirst)!0.5!(\tikzinputsegmentlast)$)
        -- (\tikzinputsegmentlast);
    }
  }, decorate
}}
\tikzstyle{vertex} = [circle, draw=black, fill=black, scale= 0.5]
\usetikzlibrary{decorations.pathmorphing}
\tikzset{snake it/.style={decorate, decoration=snake}}
\begin{document}

\title{A quest for a fair schedule: The Young Physicists' Tournament\thanks{
Katar\'\i na Cechl\'arov\'a was supported by the  APVV-17-0568 from the Slovak Research and Development Agency. \'{A}gnes Cseh was supported by the Hungarian Academy of Sciences under its Momentum Programme (LP2016-3/2020), OTKA grant K128611, and COST Action CA16228 European Network for Game Theory. Zsuzsanna Jank\'{o} was supported by OTKA grant K128611. Mari\'an Kire\v s was supported by VEGA grant 1/0265/17. }
}

\titlerunning{Fair schedule at IYPT}        

\author{Katar\'\i na Cechl\' arov\'a \and \'Agnes Cseh \and Zsuzsanna Jank\'o \and
       Mari\' an Kire\v s \and Luk\'a\v s Mi\v no
}


\institute{K. Cechl\' arov\'a \at
              Institute of Mathematics, Faculty of Science, P. J. \v Saf\' arik University, Ko\v sice, Slovakia\\
              \email{katarina.cechlarova@upjs.sk}        
          \and
            \'A. Cseh \at
              Hasso Plattner Institute, University of Potsdam, Potsdam, Germany\\
Institute of Economics, Centre for Economic and Regional Studies, Budapest, Hungary\\
              \email{agnes.cseh@hpi.de}        
           \and
           Z. Jank\'o \at
              Department of Operations Research and Actuarial Sciences, Corvinus University of Budapest, Budapest, Hungary\\
Institute of Economics, Centre for Economic and Regional Studies, Budapest, Hungary\\
              \email{zsuzsanna.janko@uni-corvinus.hu}        
           \and
           M. Kire\v s \at
              Institute of Physics, Faculty of Science, P. J. \v Saf\' arik University,  Ko\v sice, Slovakia\\
              \email{marian.kires@upjs.sk}        
           \and
           L. Mi\v no \at
              Center for Information Science and Information Technologies,  Technology and Innovation Park, P. J. \v Saf\' arik University, Ko\v sice, Slovakia\\
              \email{lukas.mino@upjs.sk}               
}

\date{Received: date / Accepted: date}

\maketitle

\begin{abstract}
The Young Physicists Tournament is an established team-oriented scientific competition between high school students from 37 countries on 5 continents. The competition consists of scientific discussions called Fights. Three or four teams participate in each Fight, while rotating the roles of Presenter, Opponent, Reviewer, and Observer among them. 
 
The rules of a few countries require that each team announce in advance 3 problems they will present at the national tournament. The task of the organizers is to choose the composition of Fights in such a way that each team presents each of its chosen problems exactly once and within a single Fight no problem is presented more than once. Besides formalizing these feasibility conditions, in this paper we formulate several additional fairness criteria for tournament schedules. We show that the fulfillment of some of them can be ensured by constructing suitable edge colorings in bipartite graphs. To find fair schedules, we propose integer linear programs and test them on real as well as randomly generated data.

\end{abstract}


\section{Introduction}
\label{sec:intro}

Teams of high school students have been competing annually at the International Young Physicists’ Tournament (IYPT for short), sometimes referred to as Physics World Cup, since 1988. Each year the international jury publishes a set of 17 problems. In the preparation phase that takes several months, teams can use any resources  to solve the problems theoretically and/or experimentally and to prepare a carefully polished presentation of the results they obtain. The competition culminates in regional, national, and international tournaments that are organized in several rounds of small scientific workshops, called Fights. During a Fight, students practice how to lead scientific discussion, ask questions and evaluate the work of their adversaries by taking the roles of a Presenter, an Opponent, 
a Reviewer, and, occasionally, an Observer. Detailed information about the exact rules, schedule, past problems, winners, etc., can be found on the international webpage {\tt  http://iypt.org} and on the webpages of national committees.

The rules of the international final state that a team can challenge another team  to present a solution of any of the 17 published problems, but for their national and regional tournaments, each of the participating 37 countries can set the rules on their own. In several countries (Austria, Germany, Slovakia, Switzerland), a local tournament consists of three rounds so that  each team participates in exactly three Fights, and in each of these Fights it presents a different problem from the set of three problems it has chosen
in advance.
Some countries formulate 
additional conditions for the schedules of the tournaments.
For example, the German rules explicitly state that the schedule has to take into account the following criteria, with decreasing priority: (1) no two teams from the same school (center) compete within one Fight, (2) no team has the same Opponent more than once, (3) if possible, each team competes with 6 different teams in its 3 Fights in the tournament.

The authors of the present paper have been contacted by members of the Slovak organizing committee who felt that besides guaranteeing the fulfillment of the necessary criteria stated in the international rules, it is desirable to create comparable conditions for all the participants, to ensure their equal treatment. The first aim of this paper is to formally define the necessary (feasibility) constraints for the schedule of an IYPT tournament. Then we formulate several fairness conditions, proposed by the organizers of local tournaments.  On the theoretical side, we draw a connection between feasible and fair schedules and edge colorings of graphs.  On the practical side and to construct fair schedules we propose several  integer linear programs and test them on real and randomly generated data.

\subsection{Related work}

Scheduling problems appear in real life, often connected with the construction of timetables at schools or schedules of sports matches. They are also a popular research topic in Mathematics and Computer Science. Many variants of scheduling problems are difficult to solve in practice even for small instances. Also, scheduling problems were among the first problems proven to be computationally hard theoretically \citep{Ull75,EvenItaiShamir1975}. In solving scheduling problems many different approaches have been used, among them variants of graph coloring problems \citep{Lewis2011, Januario2016}, integer programming \citep{BriskornDrexl, Atan2018}, constraint programming \citep{BLN12},  application of SAT encoding \citep{SAT2014}, and various heuristic algorithms, such as ant colony optimization \citep{Lewis2011}. 

Fairness in connection with scheduling appears in different contexts. Here we review progress on the study of fair schedules in the three most relevant fields to our study: work shifts, timetables, and sports tournaments. Finally, we argue why student competitions should become a fourth point on the list of practical scenarios where the computation of a fair schedule is essential.

\paragraph{Work shifts} The shift scheduling problem involves determining the number of employees to be assigned to each shift and specifying the timing of their relief and breaks, while minimizing the total staffing cost and the number of employees needed \citep{Edi54,Ayk96}. Recent advances on the topic move into the direction of fairness. \citet{SB12} minimize the paid out hours under the restrictions given by the labor agreement, and, subject to this, they also integrate the preferences of laborers and fairness aspects into the scheduling model. \citet{BD14} construct a flexible MIP framework to satisfy all service requirements and contractual agreements, while respecting workers’ preferences on 
workload balancing.


\paragraph{Timetables} A widespread application of timetable design is creating a timetable for students and teachers in a school, so that it satisfies as many wishes as possible while guaranteeing that all demands regarding subjects, rooms, and working hours are satisfied. The \citet{EUR19} maintains a constantly updated list of research papers on educational timetabling. Automated timetabling has various applications outside schools as well \citep{Sch99}. In a recent paper, \citet{Vangerven2018} construct a schedule for a conference with parallel sessions that, based on preferences of participants, maximizes total attendance and minimizes session hopping.

\paragraph{Sports tournaments} Fairness plays an essential role in sports tournament scheduling \citep{DK07, BK10,VG20}. In spite of the relevance of good game schedules, very few professional leagues have adopted optimization models and software to date \citep{Ras08,NGB+10,GS12}. One of these exceptions is the national soccer tournament in Brazil. \citet{UR09} designed an ILP-based system, which was used for the first time in 2009 as the official scheduler to build the fixtures of the first and second divisions as well. Their solution minimizes the number of breaks and maximizes the number of games that open TV channels could broadcast. In works dealing with the scheduling of round robin tournaments, fairness criteria appear in the form of balancing the number of cases when teams play two consecutive  home or away games, balancing the time after the most recent game of two opposing teams, or balancing the difference between the number of games played by any two teams at any point in the schedule \citep{Miyashiro2005, Suksompong2016, Atan2018}.

\paragraph{Student competitions} IYPT has its counterpart in mathematics, the International Tournament of Young Mathematicians (ITYM), which has a similar tournament structure, with teams playing the roles of the Presenter, Opponent, Reviewer, and Observer. Another branch of student competitions organized in rounds in which teams take turns are debating tournaments \citep{Debating, Bra17}. The World Universities Debating Championship is the world's largest debating tournament and one of the largest annual international student events in the world. At their events, the British Parliamentary format is used, in which four teams participate in each round~\citep{DEB14}. Two teams form the ``government'' and the other two the ``opposition'' in each debate room, and the order of speeches assigns a different role to each of the teams. Such competitions promote democratic education and they are shown to significantly enhance student performance in the subject, hence they are currently on the rise~\citep{Spi07,PWL18}.

Compared to sports tournaments, scheduling competitions for students is an admittedly much less profitable, but highly noble branch of tournament scheduling. Up to our knowledge, no formal scheduling model for organizing student competitions has been reported on yet. In this work, we make an attempt to demonstrate how students' competitions can be organized with the aid of integer programming, which not only automatizes the cumbersome task of scheduling, but also calculates a solution that is provably more fair for the participating students. 

\subsection{Outline}

In Section~\ref{sec:details} we outline the rules and organization of the IYPT in more detail and in Section~\ref{s_notation} we formally introduce the studied problem and the related notions. Section~\ref{s_graphs} is devoted to a discussion of how the edge coloring of bipartite graphs leads to a feasible simple schedule and to schedules that give each team 3 different order positions in its 3 Fights.
 We formulate several fairness criteria for schedules; as far as we know, fairness criteria similar to ours have not been considered before in scheduling problems. In Section~\ref{s_ILP} we formulate integer linear programs to find fair schedules fulfilling alternative---weaker and stronger---forms of fairness. Then, in Section~\ref{s_comp} we report on the results we obtained when the designed ILPs were applied to real data: we used the application sets from regional tournaments in Slovakia in recent years. We also randomly generated sets of applications that have some features similar to the expected situations and performed numerical tests on these random data. 

\section{Background}\label{sec:details}
 
According to the rules of the Austrian, German, Slovak, and Swiss regional and national tournaments, each team applying for participation announces a subset of exactly 3 problems from the published set of 17 problems. This subset is called the team's {\it portfolio} and it contains the 3 problems the team will present at the tournament. A set of portfolios may look similar to the one presented in Table~\ref{t_appl}, which is a real set of portfolios from the regional tournament Bratislava 2018.

\begin{table}[ht]
{\small
\begin{center}
\begin{tabular}{|cl|cl|cl|cl|}
	      \noalign{\hrule}
 Team  & Portfolio     & Team  & Portfolio & Team  & Portfolio   & Team  & Portfolio \\
	      \noalign{\hrule}
Sharks1  & 4,6,14      & Whales1 & 3,7,14   & Turtles1 & 2,3,14   & Eagles & 4,9,16\\
Sharks2 & 10,16,17     &Whales2 & 2,5,12    & Turtles2 & 5,6,10   & Lions & 4,9,10 \\
Sharks3 & 1,7,13  &Whales3 & 4,9,10    & Bears1 & 3,4,8    & Dogs & 3,4,7\\
&&&   & Bears2 & 5,9,17    & &\\
      \noalign{\hrule}
    \end{tabular}
\end{center}
\caption{The set of portfolios in the regional tournament Bratislava 2018. We use this example instance throughout the entire paper. The participating teams are anonymized by having been given animal names. To indicate which teams are from the same school we use the name of the same animal and distinguish the different teams only by the final digit.}\label{t_appl}
}	
\end{table}

The tournament is organized in 3 rounds. In each round, the set of teams is partitioned into rooms, each of which hosts a so-called {\it Fight}.
The number of teams participating in a Fight is 3 or~4, and the number of stages in a Fight is also 3 or 4, respectively. Now we describe the structure of a Fight.

The assignment of teams to rooms in the rounds  also specifies which team will present which problem
from their portfolio.
Suppose that the set of teams in a room is 
A, B, C (see Table~\ref{tab:roles})
and assume that these teams have been assigned problems $p_A$, $p_B$, and $p_C$, respectively, to present. In the first stage of the Fight, team~A is the Presenter; it delivers a report on problem~$p_A$. Team~B is the Opponent. After the report
of the Presenter team, the Opponent team evaluates the report, stressing its pros and cons. Afterwards the third team~C, the Reviewer, can ask questions  both other teams and then the Reviewer presents an overview of the performance of the Opponent. The stage ends by the Presenter stating  
some concluding remarks. Finally the jury may ask some short questions to all three active teams. After a short break, another stage with the same structure begins, but the roles of teams are rotated, as illustrated by Table~\ref{tab:roles}.
This means that in stage two, team~B is the Presenter, team~C is the Opponent and team~A is the Reviewer; in stage three team~C is the Presenter, team~A is the Opponent and team~B is the Reviewer. Hence, each team performs each role during a Fight exactly once.

\begin{table}[ht]
\begin{center}
    \resizebox{\textwidth}{!}{
\begin{tabular}{|c|c|c|c|}
	      \noalign{\hrule}
	        &
\multicolumn{3}{c|}{Stage}\\
Team & 1 & 2 & 3 \\
 \noalign{\hrule}
 A  & Pres.  & Rev. & Opp. \\
 B  &  Opp.  & Pres. & Rev.     \\
 C & Rev. &   Opp.  & Pres. \\
	      \noalign{\hrule}
    \end{tabular}
    \hskip1cm
\begin{tabular}{|c|c|c|c|c|}
	      \noalign{\hrule}
	        &
\multicolumn{4}{c|}{Stage}\\
Team & 1 & 2 & 3 & 4 \\
 \noalign{\hrule}
 A  & Pres.  & Obs. & Rev. & Opp. \\
 B  &  Opp.  & Pres. & Obs. & Rev.     \\
 C & Rev. &   Opp.  & Pres.& Obs. \\
D & Obs. & Rev. &   Opp.  & Pres. \\

	      \noalign{\hrule}
    \end{tabular}
    }
\end{center}
\caption{Schemes for 3- and 4-team Fights, extracted from the official regulations of the IYPT~\citep{IYPT17}.}\label{tab:roles}
\end{table}

If the total number of teams is not divisible by three or if the organizers have some other issues to deal with (e.g., there are not enough rooms on the premises where the tournament takes place, or the number of available qualified jurors is small, etc.), the number of teams in a room may be 4. In such a Fight, the 4 teams also exchange their roles cyclically (see the right hand-side of Table~\ref{tab:roles}), with one of them playing the role of the Observer, which is the team not participating actively in the given stage.

Given the set of portfolios, an important task of the organizers is to prepare a schedule of the tournament. For each team, the schedule specifies the problem, the room, and the stage for each of the three rounds. For example, the schedule depicted in Table~\ref{t_real_schedule} instructs team Sharks1 to present problem~4 in room 1 in the first round as the Presenter in the first stage in that Fight. According to the same schedule, Sharks1 will present problem~6 in room 4 in the second round, and it will be the third team to present a problem in that Fight; and, finally, it will present problem~14 in room 3 in the third round, again as the third Presenter in that Fight. Such an assignment clearly determines the course of the whole tournament. 

Each schedule has to fulfill the following obvious conditions.

\begin{enumerate}[(1)]
    \item[$(a)$] Each team presents exactly the 3 problems from its portfolio.
    \item [$(b)$] No problem is presented more than once during the same Fight.
    \item [$(c)$] In each Fight, the correct number of problems (3 or 4) is presented.
    \item [$(d)$] For each Fight, an ordering of Presenters is defined.
    \end{enumerate}
A schedule fulfilling conditions $(a)-(d)$ is said to be {\it feasible}. In Section~\ref{s_graphs} we will see that feasible schedules are guaranteed to exist under very mild conditions. 

A usual requirement of the organizers is to group the teams into Fights so that all participating teams in a Fight come from different schools. Besides avoiding bias and possible help between teams from the same school, such \textit{non-cooperative} schedules encourage scientific interaction between students who have not met before.

Recall now the cyclic exchange of roles of teams within a Fight, defined in Table~\ref{tab:roles}.  A team may feel to be put at a disadvantage 
if it plays the role of team~A in all its Fights, because then it has to start each of these Fights as the first Presenter.
So we introduce another fairness notion: we say that a schedule is {\it order fair} if each team has three different order positions in its three Fights.

The need for our most important fairness notion only arose in 2019, as a second step of making the Slovak regional rounds of IYPT more student-centered. A few years earlier, new members had joined the local organizing committee of the IYPT. These colleagues, originally physics teachers, who also served as jury members before, were familiar with the situation in teaching physics in local secondary schools and also with students' needs. Their first initiative was to establish the system that allowed teams to choose three problems.\footnote{Previously, Slovak regional tournaments were organized according to the international rules where during a fight a team could challenge another team with any of the 17 problems published by the international jury. The new committee members told us about several cases when students were deeply disappointed when  they lost the competition because some other team played strategically and challenged them by a problem they were not prepared for, as they put more effort into another problem or they could not solve it because their school did not have the necessary experimental equipment unlike the school of their adversaries. These colleagues felt that it is more important to stress the motivational aspect of the competition and let students experience the pride in being able to present the work they had done in preparation for the tournament. So they suggested the change of rules in that each team presented only the problems chosen in advance. But, when these rules were applied, they also noticed the irregularities mentioned above.} The new rules brought new challenges, and in 2019, the same organizers contacted the authors of the present paper with the request to help to formalize fairness conditions and design an automated procedure for computing schedules meeting these criteria.

We now explain the most challenging 
fairness concern 
for schedules on an intuitive level and by an example. Assume that teams $t_i$ and $t_j$ are in the same Fight and team $t_i$ presents problem~$p$. If team $t_j$ has problem $p$ in its portfolio too, then it has either presented $p$ before in a previous round or will present it in some later round. In the former case, team $t_j$ had prepared its own presentation for $p$, moreover, it has already heard the comments of its own Opponent and Reviewer on problem $p$, so now team $t_j$ is likely to be better prepared for the tasks of the Opponent as well as of the Reviewer. In the latter case, team $t_j$ has a chance to update its own presentation based on what it has heard during the presentation of problem $p$ by team $t_i$ and also be better prepared for answering the challenges of its future Opponent and Reviewer on problem $p$. The organizers wish to avoid that such injustice happens. 

We say that a feasible schedule is {\it fair} if the following condition for each pair of teams $t_i,t_j$ is fulfilled: If teams $t_i,t_j$ are 
in the same Fight at any time during the tournament and team $t_i$ presents problem $p$ in this Fight, then problem $p$ is not in the portfolio of team~$t_j$. 

In reality, it has not always been the case that the used schedules fulfilled the fairness requirements. Table~\ref{t_real_schedule} depicts the real schedule of the regional tournament Bratislava 2018, corresponding to the set of portfolios from Table~\ref{t_appl}. 
Have a look at team Lions. In the first round, 
it presents problem~9 and sees team Sharks1 presenting problem~4 in the same Fight. In the second round, team Lions presents problem~4 and sees team Sharks2 presenting problem~10. In the final round, team Lions presents problem~10.
This means that team Lions had seen two problems from its portfolio, namely problems~4 and~9, before it had to present them. This is clearly unfair, as team Lions had a great advantage to other teams. For the set of portfolios in this regional tournament a fair schedule exists, and it is presented in Table~\ref{t_fair_schedule}.

\begin{table}[ht]
\centering
    \resizebox{\textwidth}{!}{
\begin{tabular}{|cl|lc|lc|lc|lc|}
   \noalign{\hrule}
  &&
\multicolumn{2}{c}{Room 1} & \multicolumn{2}{|c}{Room 2} & \multicolumn{2}{|c}{Room 3} & \multicolumn{2}{|c|}{Room 4} \\
	      \noalign{\hrule}
&&Team & Problem & Team & Problem & Team & Problem   & Team & Problem \\
   \noalign{\hrule}
\parbox[t]{2mm}{\multirow{4}{*}{\rotatebox[origin=c]{90}{{\bf Round 1}}}}
&A &Sharks1  & 4       & Whales2 & 2      &  Bears1  & 4   &  Whales3 & 10  \\
&B &Turtles1  & 3       & Sharks3 & 7      &  Whales1  & 14     &  Dogs & 3  \\
&C &Lions  & 9           & Eagles & 9        &  Turtles2  & 10      &  Bears2 & 5  \\
&D  &  &                  &&                   &&                          & Sharks2 & 17         \\    
   \noalign{\hrule}
& &
\multicolumn{2}{c}{Room 1} & \multicolumn{2}{|c}{Room 2} & \multicolumn{2}{|c}{Room 3} & \multicolumn{2}{|c|}{Room 4} \\
	      \noalign{\hrule}
&&Team & Problem & Team & Problem & Team & Problem   & Team & Problem \\
   \noalign{\hrule}
\parbox[t]{2mm}{\multirow{4}{*}{\rotatebox[origin=c]{90}{{\bf Round 2}}}}
&A &Lions  & 4        & Turtles1 & 14      &  Dogs  & 7           &  Eagles & 4  \\
&B &Sharks2  & 10       & Whales3 & 4      &  Bears2  & 9     &  Whales1 & 7  \\
&C &Bears1  & 3 & Sharks3 & 1        &  Whales2  & 12      &  Sharks1 & 6  \\
&D &                   &&                     &&                          & &  Turtles2 &  5          \\    
   \noalign{\hrule}
&&
\multicolumn{2}{c}{Room 1} & \multicolumn{2}{|c}{Room 2} & \multicolumn{2}{|c}{Room 3} & \multicolumn{2}{|c|}{Room 4} \\
	      \noalign{\hrule}
&&Team & Problem & Team & Problem & Team & Problem   & Team & Problem \\
   \noalign{\hrule}
\parbox[t]{2mm}{\multirow{4}{*}{\rotatebox[origin=c]{90}{{\bf Round 3}}}}
&A & Bears2  & 17        & Sharks2 & 16      &  Turtles2  & 6           &  Sharks3 & 13  \\
&B &Whales2  & 5            & Dogs & 4      &  Eagles  & 16                   &  Bears1 & 8  \\
&C &Turtles1  & 2           & Whales1 & 3        &  Sharks1  & 14           &  Whales3 & 9  \\
&D &&                    &&                     &&                                         &  Lions &   10       \\    
   \noalign{\hrule}
\end{tabular}
}
\caption{The real  schedule used in the regional tournament Bratislava 2018.}\label{t_real_schedule}
\end{table}

Notice further that the schedule in Table~\ref{t_fair_schedule} is also unbalanced in another way. Team Sharks1 has to oppose or review 6 different problems during the tournament, namely problems 2, 3, 7, 9, 10, and 17. By contrast, team Turtles1 opposes or reviews only four problems: 1, 4, 5, and 7. Clearly, this gives Turtles1 another form of advantage to team Sharks1. We will say that a feasible schedule is {\it strongly fair} if each team deals with each problem (in any role) during the tournament at most once. 

\begin{table}[ht]
\centering
    \resizebox{\textwidth}{!}{
\begin{tabular}{|cl|lc|lc|lc|lc|}
   \noalign{\hrule}
& & 
\multicolumn{2}{c}{Room 1} & \multicolumn{2}{|c}{Room 2} & \multicolumn{2}{|c}{Room 3} & \multicolumn{2}{|c|}{Room 4} \\
	      \noalign{\hrule}
&&Team & Problem & Team & Problem & Team & Problem   & Team & Problem \\
   \noalign{\hrule}
\parbox[t]{2mm}{\multirow{4}{*}{\rotatebox[origin=c]{90}{{\bf Round 1}}}}
&A&Sharks1  & 6       & Lions & 9                &  Sharks3  & 1   &  Bears1 & 8  \\
&B&Whales1  & 3       & Sharks2 & 16         &  Whales3  & 4     &  Turtles2 & 10  \\
&C&Bears2  & 9    & Whales2 & 12        &  Turtles1  & 3      &  Eagles & 9 \\
&D&&                      &&                     &&                          & Dogs & 7         \\    
 \noalign{\hrule}
&& 
\multicolumn{2}{c}{Room 1} & \multicolumn{2}{|c}{Room 2} & \multicolumn{2}{|c}{Room 3} & \multicolumn{2}{|c|}{Room 4} \\
	      \noalign{\hrule}
&&Team & Problem & Team & Problem & Team & Problem   & Team & Problem \\
   \noalign{\hrule}
\parbox[t]{2mm}{\multirow{4}{*}{\rotatebox[origin=c]{90}{{\bf Round 2}}}}
&A&Sharks1  & 4        & Sharks2 & 10      &  Lions  & 4           &  Whales3 & 10  \\
&B&Whales1  & 7       & Whales2 & 2      &  Sharks3  & 7       &  Bears2 & 5  \\
&C&Turtles1  & 2        & Bears1 & 4   &  Turtles2  & 6       &  Eagles & 16  \\
&D &  &                      &&                     &&                          & Dogs & 3         \\    
\noalign{\hrule}
&& 
\multicolumn{2}{c}{Room 1} & \multicolumn{2}{|c}{Room 2} & \multicolumn{2}{|c}{Room 3} & \multicolumn{2}{|c|}{Room 4} \\
	      \noalign{\hrule}
&&Team & Problem & Team & Problem & Team & Problem   & Team & Problem \\
   \noalign{\hrule}
\parbox[t]{2mm}{\multirow{4}{*}{\rotatebox[origin=c]{90}{{\bf Round 3}}}}
&A&Lions  & 10                & Sharks2 & 17      &  Whales2  & 5           &  Sharks3 & 13  \\
&B&Sharks1  & 14            & Whales1 & 14      &  Turtles1  & 14         &  Whales3 & 9  \\
&C&Bears2  & 17           & Eagles & 4        &  Dogs  & 4           &  Bears1 & 3  \\
&D &&                    &&                     &&                                         &  Turtles2 &  5       \\    
   \noalign{\hrule}
\end{tabular}
}
\caption{A fair schedule for the regional tournament Bratislava 2018. Since no Fight contains two teams from the same school, this schedule is non-cooperative. Strong fairness does not hold; e.g., team Whales1 deals with problem~4 in Round~2 and Round~3 as well. Team Lions plays role A in all its Fights, thus the schedule is not order fair.}\label{t_fair_schedule}
\end{table}

\section{Notation and optimality concepts}\label{s_notation}

We start this section with introducing the notation used thorough this paper and formalizing the feasibility requirements for a schedule. In Section~\ref{s_feasible3}, we define two optional features of feasible schedules, which can be enforced individually and on the top of feasibility, if the decision maker finds them desirable. Then we proceed to formalize the three degrees of fairness in Section~\ref{s_fair}.

$T=\{t_1,\dots,t_n\} $ is a set of $n$ {\bf teams} with a partition ${\cal T}=\{T_1,T_2,\dots,T_\Lambda\}$, where the partition sets are called {\it schools}. $P=\{p_1,\dots,p_m\}$ is a set of $m$ {\bf problems}. 

Each team $t$ applies with a set of exactly 3 problems from set $P$; these three problems will be called the {\bf portfolio} of team $t$ and denoted by~$P(t)$. The {\bf profile} is an $n$-tuple of portfolios $\Pi = (P(t_1), P(t_2),\dots, P(t_n))$. 
For a given $S\subseteq T$, we denote by $P(S)$ the set of problems that appear in the portfolio of at least one team from $S$, that is,
$P(S)=\cup_{t\in S} P(t)$. If $p\notin P(t)$ for team $t\in T$ and problem $p\in P$, then we say that team $t$ {\bf avoids} problem~$p$.  
 
There are $s$ {\bf rooms} $R=\{r_1,\dots,r_s\}$. The set of rooms is partitioned into two subsets $R_3$ and~$R_4$. If $r\in R_3$ then room $r$ hosts 3-Fights (i.e., exactly three teams perform a Fight in $r$); if $r\in R_4$ then room $r$ hosts 4-Fights (Fights of 4 teams). The size of room $r$ is denoted by~$\size(r)$. Obviously, $\size(r)=3$ for $r\in R_3$ and $\size(r)=4$ if $r\in R_4$. 

An instance of IYPT is a tuple $(T,P,\Pi,R)$, denoting the teams, the problems, the profile, and the rooms. Now we describe the output, denoted by $(\mathcal{P}, \mathcal{R},\mathcal{O})$. For an integer $k$, the notation $[k]$ represents the set $\{1,2,\dots,k\}$. 

There are 3 \textbf{rounds}, and a \textbf{Fight} is uniquely defined by the pair $(j,r)$, where $j$ is a round and $r$ is a room. For each of the 3 rounds, each team needs to be assigned the problem it will present, the room in which this presentation will take place, and its stage of presentation within the Fight. We formalize this as follows. The collection $\mathcal{P}=\{\pi_j:T\to P; \ j\in[3]\}$ consists of functions $\pi_1, \pi_2, \pi_3$, where each of these three functions maps exactly one problem to each team. The value  $\pi_j(t)$ for a given team $t$ is the problem $t$ will present in round~$j$. Another collection of functions is $\mathcal{R}=\{\rho_j: T \to R; \ j\in[3]\}$; 
where $\rho_j(t)$ for
 a given team $t$ specifies the room $t$ is assigned to for round~$j$. Finally, $\mathcal{O}=\{ \omega_j: T \to \{A,B,C,D\};\ j\in[3]\}$ is a collection of three functions that map each team to an element in the order set $\{A,B,C,D\}$. For 
$t\in T$, function $\omega_j(t)$ tells in which stage team $t$ will be the Presenter in round~$j$. To simplify notation, capital letters A, B, C, and D will be reserved for denoting that a team plays the role of the Presenter in stage 1, 2, 3, and 4, respectively, within a Fight.

Now we are ready to define a feasible schedule formally. 

\begin{df}\label{def:feasible} 
A {\bf feasible schedule} is a triple $(\mathcal{P}, \mathcal{R},\mathcal{O})$ 
where $\mathcal{P}=\{\pi_j:T\to P; \ j\in[3]\}$,
$\mathcal{R}=\{\rho_j:T\to R; \ j\in[3]\}$ and
$\mathcal{O}=\{ \omega_j:T\to \{A,B,C,D\};\ j\in[3]\}$ 
are mappings of teams to problems, rooms, and order set $\{A,B,C,D\}$,
respectively, such that
\begin{enumerate}[(i)]
\item $\{\pi_1(t),\pi_2(t),\pi_3(t)\}=P(t)$ for each team $t\in T$;
\item if $\rho_j(t)=\rho_j(t')$ then $\pi_j(t)\ne\pi_j(t')$ for round $j$ and each pair of different teams $t,t'\in T$;
\item $|\{t\in T\ :\ \rho_j(t)=r\}|=\size(r)$ for each round $j$ and each room $r\in R$;
\item $\{ \omega_j(t)\ :\  \rho_j(t)=r\}=\{A,B,C\}$ for each $j\in[3]$ and $r\in R_3$ and \\ $\{ \omega_j(t)\ :\  \rho_j(t)=r\}=\{A,B,C,D\}$ for each $j\in[3]$ and~$r\in R_4$.
\end{enumerate}
\end{df}
As described above, the interpretation of the mappings in Definition~\ref{def:feasible} is such that $\pi_j(t)$ denotes the problem presented by team $t$ in round $j$, $\rho_j(t)$ denotes the room to which
team $t$ is assigned in round $j$, and $ \omega_j(t)$ corresponds to the order of team $t$ in round~$j$. Condition $(i)$ then ensures that each team presents exactly the problems from its portfolio during the tournament; condition $(ii)$ means that in no Fight the same problem is presented more than once; condition $(iii)$ ensures the correct number of teams for each room, i.e., this should be equal to the size of the respective room; and finally, condition $(iv)$ makes sure that exactly one team is chosen to be the Presenter in each stage of a Fight. 
These points correspond to the requirements listed in Section~\ref{sec:details}.


\subsection{Refinement of feasible schedules}\label{s_feasible3}

To avoid cooperation of teams from the same school, a schedule might be required to prevent that two teams from the same school participate in the same Fight. We remind the reader that 
one partition subset from $\mathcal{T}$
corresponds to the set of teams from the same school.

\begin{df}\label{def_non-coop}
A schedule is {\bf non-cooperative} if it is feasible and
$$\rho_j(t)\ne \rho_j(t') \mbox{\ for each \ } j\in[3]$$
whenever $t$ and $t'$ belong to the same partition subset $T_i \in \mathcal{T}$, $i \in [\Lambda]$.
\end{df}

\noindent
The following definition ensures that no team has the same ordering position (A, B, C, D) in two Fights it participates in.
\begin{df}\label{def_order_fair}
A schedule is {\bf order fair} if it is feasible and
$|\{ \omega_1(t), \omega_2(t), \omega_3(t)\}|=3$
for each $t\in T$.
\end{df}

\subsection{Fairness properties}\label{s_fair}

The most striking problem with feasible schedules is that certain teams have considerable advantage to others, if they repeatedly encounter the problems in their own portfolio. In the following, we define 3 degrees of fairness based on restrictions applied to what presentations a team can witness. The condition that no team can see a presentation of a problem in its portfolio by some other team is captured by Definition~\ref{def_fair}.

\begin{df}\label{def_fair}
A schedule is {\bf fair} if it is feasible and the following condition
holds for all rounds $j\in[3]$:
\begin{equation}\label{df_fair}
\mbox{if\ } \rho_j(t)=\rho_j(t') \mbox{\ for two different teams\ }t,t'\in T \mbox{\ and\ } \pi_j(t)=p, \mbox{\ then \ } p\notin P(t'). 
\end{equation}
\end{df}

In some cases, a fair schedule does not exist or it cannot be computed. For these cases, the organizers suggested to `sacrifice' the fairness of the last round. Enforcing the fairness condition (\ref{df_fair}) for the first two rounds only  ensures that no team presents a problem after having watched some other team presenting the same problem, which would be a hard violation of fairness. However, it allows a team to play the role of the Opponent or the Reviewer for a problem it presented in an earlier round. This is clearly a milder violation of fairness. The following definition captures this relaxation of fairness.

\begin{df}\label{def_wfair}
A schedule is {\bf weakly fair} if it is feasible and condition (\ref{df_fair})
holds for rounds~$j=1,2$.
\end{df}

To define a stronger form of fairness, let us introduce the following notation that represents the set of problems that team $t$ deals with in round $j$ in any role (Presenter, Opponent, Reviewer, or, in case of 4-rooms, Observer).
$$P(j,\rho_j(t))=\{p\in P\ : \mbox{\ there exists a team $t'\in T$ such that \ }  \rho_j(t)=\rho_j(t') \ \mbox{and}\ \pi_j(t')=p\}$$

\begin{df}\label{def_sfair}
A schedule is {\bf strongly fair} if it is feasible and for each team $t\in T$ the following holds:
\begin{equation}\label{e_sfair}
|\{P(1,\rho_1(t))\cup P(2,\rho_2(t))\cup P(3,\rho_3(t))\}| = \size(\rho_1(t))+\size(\rho_2(t))+\size(\rho_3(t)).
\end{equation}
\end{df}

\noindent In other words, Definition~\ref{def_sfair} means that no two problems a team $t$ deals with during the tournament are identical. In particular, if $p\in P(t)$ and team $t$ can see the presentation of problem $p$ in some Fight, then this implies that team $t$ deals with $p$ at least twice (the other occasion is when $t$ presents $p$) and hence condition (\ref{e_sfair}) is violated for team~$t$. Therefore we have the following relation between fairness notions.

\begin{obs}
Each strongly fair schedule is fair and each fair schedule is weakly fair. 
\end{obs}

\section{Feasible solutions via graph coloring}\label{s_graphs}

In this section we utilize combinatorial tools to derive positive results for feasible schedules. With the help of edge colorings and basic theorems in matching theory, we characterize the existence of so-called simple solutions in Section~\ref{s_simple}, and give a constructive algorithm to compute an order fair schedule in Section~\ref{s_order}.

First, we recall some basic notions of the graph theory  used in this section. A graph is a pair $G=(V,E)$, where $V$ is a set of vertices and $E$ is a set of edges, i.e., pairs of vertices. We say that $G$ is bipartite, if the vertex set $V$ can be partitioned into two sets so that each edge connects two vertices from different partitions. If $e=\{u,v\}$ is an edge, we say that edge $e$ is incident to  vertices $u$ and $v$, and that vertices $u,v$ are adjacent.  The number of edges incident to a vertex $v$ is called its degree and is denoted by $\deg(v)$; while $\Delta(G) = \max_{v \in V(G)}{\deg(v)}$ is the maximum degree in graph~$G$. 
The neighborhood of a vertex set $U\subseteq V$ in $G$ is the set $N(U)$ of all vertices that are adjacent to vertices in~$U$.  For a set of vertices U, we denote by~$G(U)$ the  subgraph of $G$ containing  $U$, all the edges incident to vertices in $U$, and the other vertices incident to these edges. A function $c:E\to {\mathbb N}$ is an {\bf edge coloring} of $G$ if $c(e)\ne c(e')$ whenever two edges $e,e'$ share a common vertex. A matching in $G$ is a set of edges $M$ such that no two edges in $M$ have a vertex in common. For a thorough review on these and other basic graph-theoretic notions, we recommend consulting the book of~\cite{Diestel}.

\subsection{Simple solutions}\label{s_simple}

In this section we construct feasible schedules that keep the composition of the teams in each room fixed in all three rounds. Such schedules will be called simple.

\begin{df}\label{d_simple}
A schedule is {\bf simple} if it is feasible and
$\rho_1(t)=\rho_2(t)=\rho_3(t)$ for each team $t\in T$.
\end{df}

 The drawback of a simple schedule is that the students can only meet and exchange ideas with a very small subset of other participants. Thus, if possible, simple schedules should be avoided in reality. In this paper we only use this concept to ensure that a feasible schedule always exists if some very mild conditions are fulfilled.

For a set of portfolios $\Pi$ and a subset of teams $S\subseteq T$ we shall denote by $G(S)$ the bipartite graph $G(S)=(S\cup P(S),E_S)$ such that the pair $\{t,p\}\in E_S$ if and only if $t\in S$ and $p\in P(t)$. Figure~\ref{fi:graph} illustrates the graph $G(T)$ for the instance from Table~\ref{t_appl}.

\begin{figure}[htb]
    \centering
    \resizebox{\textwidth}{!}{
	\begin{tikzpicture}[transform shape]
	
	    \tikzstyle{edge} = [very thick]
        
		\pgfmathsetmacro{\b}{1.25}
		\pgfmathsetmacro{\d}{\b*3}
		\def\teams{{S1,S2,S3,W1,W2,W3,T1,T2,B1,B2,E,L,D}}
    
		\foreach \i in {1,..., 17} {
		    \node[vertex,label=below:$p_{\i}$] (p\i) at ({\b *13* (\i-1)/17}, 0) {};
		}
		
		\node[vertex,label=above:S1] (t1) at ({\b * 0}, \d) {};
		\node[vertex,label=above:S2] (t2) at ({\b * 1}, \d) {};
		\node[vertex,label=above:S3] (t3) at ({\b * 2}, \d) {};
		\node[vertex,label=above:W1] (t4) at ({\b * 3}, \d) {};
		\node[vertex,label=above:W2] (t5) at ({\b * 4}, \d) {};
		\node[vertex,label=above:W3] (t6) at ({\b * 5}, \d) {};
		\node[vertex,label=above:T1] (t7) at ({\b * 6}, \d) {};
		\node[vertex,label=above:T2] (t8) at ({\b * 7}, \d) {};
		\node[vertex,label=above:B1] (t9) at ({\b * 8}, \d) {};
		\node[vertex,label=above:B2] (t10) at ({\b * 9}, \d) {};
		\node[vertex,label=above:E] (t11) at ({\b * 10}, \d) {};
		\node[vertex,label=above:L] (t12) at ({\b * 11}, \d) {};
		\node[vertex,label=above:D] (t13) at ({\b * 12}, \d) {};	
		
        \foreach \i/\j in {1/4,1/6,1/14,2/10,2/16,2/17,3/1,3/7,3/13,4/3,4/7,4/14,5/2,5/5,5/12,6/4,6/9,6/10, 7/2,7/3,7/14,8/5,8/6,8/10,9/3,9/4,9/8,10/5,10/9,10/17,11/4,11/9,11/16,12/4,12/9,12/10,13/3,13/4,13/7} {
            \draw[edge] (t\i) -- (p\j);
        }
	\end{tikzpicture}
	}
    \caption{The portfolios from Table~\ref{t_appl}, represented by the bipartite graph $G(T)=(S\cup P(T),E_T)$. The team names are abbreviated to their first letter and team number, e.g., S1 denotes Sharks1.}
    \label{fi:graph}
\end{figure}
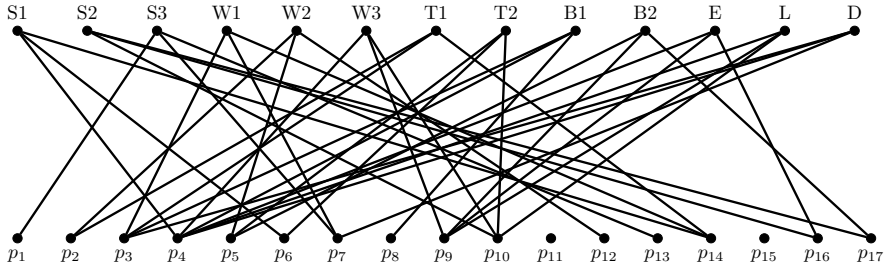

The official rules of the IYPT prefer 3-team Fights and admit 4-team Fights only if the total number of teams $n$ is not divisible by 3. In such cases,  the number of 4-team Fights, i.e., $|R_4|$, should be equal to $n$ modulo~3.  We will deal with the cases when $n$ modulo 3 is equal to 0, 1, and 2 separately.

\begin{thm}\label{t_simple}
If the number of teams $n$ is divisible by 3, then a simple schedule exists.
\end{thm}

\begin{proof}
Partition the set of teams into 3-rooms arbitrarily. The only thing to ensure a feasible schedule is to decide for each room who will present which problem in which round --- without scheduling two presentations of the same problem for the same Fight. Fix a room $r$ and assume that the three teams assigned to the three Fights to be performed in $r$ are $T(r)=\{t_1,t_2,t_3\}$. Notice that in the bipartite graph $G(T(r))$ the maximum degree of a vertex is $\Delta(G(T(r)))=3$. This is because the degrees of vertices in $T(r)$ are {\it exactly} 3 (the size of the portfolio of each team is 3) and the degrees of vertices in $P(T(r))$ are {\it at most} 3. Therefore, 
by K\"onig's theorem (\cite{Konig}, see also \cite{Diestel}, Proposition~5.3.1.), $G(T(r))$ admits an edge coloring by 3 colors. One color class corresponds to the assignment of problems to be presented by teams in one stage of the Fight.
\end{proof}

If $n$ is not divisible by 3, then we need one or two rooms with 4 teams. Now we only need to ensure that the set of portfolios contains a suitable set of 4 teams (or two disjoint quadruples of teams) that can be organized in the same room during the tournament, as the rest of teams can be dealt with according to the previous theorem. 
Notice that the assignment of problems to be presented in the three rounds in a 4-room containing the set of teams $S$ again corresponds to a 3-coloring of graph~$G(S)$. Again, by K\"onig's theorem, this is ensured if $\Delta(G(S))=3$. We will call a set of teams $S\subseteq T$ with $|S|=4$ {\bf fine} if $\Delta(G(S))=3$.

Now we discuss the case of one 4-room only.

\begin{thm} 
If the number $n$ of teams fulfills $n \equiv 1\ (\textrm{mod}\ 3)$, then
a simple schedule exists if and only if each problem $p\in P$ is avoided by at least one team.
\end{thm}
\begin{proof} As we argued above, a simple schedule exists if and only if a fine set of teams exists.
Let $t_1\in T$ be an arbitrary team and let $P(t_1)=\{p_1,p_2,p_3\}$. Let team $t_2$ be any team that avoids problem~$p_1$. Now we distinguish three cases. If $|P(\{t_1,t_2\})|=6$ then the quadruple $t_1,t_2,t_3,t_4$ is fine for any two teams $t_3,t_4$. If $|P(\{t_1,t_2\})|=5$, assume w.l.o.g.\ that $P(t_1)\cap P(t_2)=\{p_2\}$. Then choose any team $t_3$ that avoids problem $p_2$ and add an arbitrary team~$t_4$. Finally, if $|P(\{t_1,t_2\})|=4$, then
$P(t_1)\cap P(t_2)=\{p_2,p_3\}$. To get a fine quadruple, choose any team $t_3$ that avoids~$p_2$. If $t_3$ happens to avoid $p_3$ too, choose $t_4$ arbitrarily, otherwise choose $t_4$ that avoids problem~$p_3$. The other direction is straightforward: each problem adjacent to any of the four teams in the fine set $S$ is avoided by at least one of the teams in $S$, because $\Delta(G(S))=3$. All other problems are avoided by all teams in~$S$.
\end{proof}

Finally, we turn to the case of two 4-rooms.
A necessary and sufficient condition for the existence of two disjoint fine sets of teams follows from Corollary 4.2.\ of \citet{Keszegh}. To be able to formulate this assertion, let us call a set $\Pi$ of $n$ portfolios {\bf special} if it has the following structure: there are $n-3$ portfolios of the form $\{p_i,p_j,p_k\}$ for some $i,j,k\in [m]$ and the remaining 3 portfolios are of the form $\{p_i,q_1,q_2\}$, $\{p_j,q_3,q_4\}$, and $\{p_k,q_5,q_6\}$, where $q_u\notin \{p_i,p_j,p_k\}$ for each~$u\in[6]$. A special set of portfolios is illustrated by Figure~\ref{fi:special_portfolio}.

 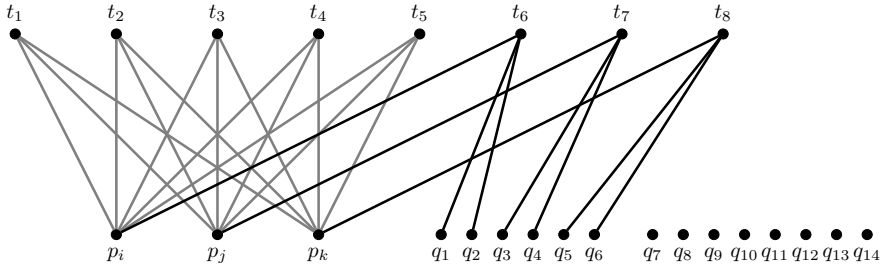
\begin{figure}[htb]
    \centering
        \resizebox{\textwidth}{!}{
	\begin{tikzpicture}[transform shape]
	
	    \tikzstyle{edge} = [very thick]
       	\pgfmathsetmacro{\b}{1.65}
		\pgfmathsetmacro{\d}{\b*2}

		\foreach \i in {1,..., 8} {
		    \node[vertex,label=above:$t_{\i}$] (t\i) at ({\b * (\i- 1)}, \d) {};
		}
		\foreach \i in {1,..., 6} {
		    \node[vertex,label=below:$q_{\i}$] (q\i) at ({\b *3 + (\i + 3)/2}, 0) {};
		}
		\foreach \i in {7,..., 14} {
		    \node[vertex,label=below:$q_{\i}$] (q\i) at ({(\b*12 + \i - 6)/2}, 0) {};
		}
		\node[vertex,label=below:$p_i$] (pi) at ({\b * 1}, 0) {};
		\node[vertex,label=below:$p_j$] (pj) at ({\b * 2}, 0) {};
		\node[vertex,label=below:$p_k$] (pk) at ({\b * 3}, 0) {};

        \foreach \i in {1,2,3,4,5} {
            \draw[edge, gray] (t\i) -- (pi);
        }
        \foreach \i in {1,2,3,4,5} {
            \draw[edge, gray] (t\i) -- (pj);
        }
        \foreach \i in {1,2,3,4,5} {
            \draw[edge, gray] (t\i) -- (pk);
        }
        \foreach \i/\j in {6/1,6/2, 7/3,7/4, 8/5,8/6} {
            \draw[edge] (t\i) -- (q\j);
        }

        \draw[edge] (t6) -- (pi);
        \draw[edge] (t7) -- (pj);
        \draw[edge] (t8) -- (pk);
        
	\end{tikzpicture}
	}
    \caption{A special profile $\Pi$ of 8 portfolios for 17 problems, out of which $q_7, \ldots q_{14}$ are not chosen by any team. This instance admits no simple schedule.}
    \label{fi:special_portfolio}
\end{figure}
 
\begin{thm} 
For the number $n$ of teams such that $n \equiv 2\ (\textrm{mod}\ 3)$ and $n\ge 8$, a simple schedule exists
if and only if the profile $\Pi$ simultaneously fulfills the following two conditions: 
\begin{enumerate}[(i)]\itemsep0pt
\item each problem is avoided by at least two teams;
\item $\Pi$ is not special.
\end{enumerate}
\end{thm}

In regional tournaments, the organizers might decide to use more 4-rooms, however, we do not have a necessary and sufficient condition for the existence of a feasible schedule in this case and leave it as an open problem. 

\subsection{Order fair solutions}\label{s_order} 

Order fairness requires that no team takes up the same ordering position in any two of its Fights. In the case of 3-rooms only, this means that each team will present one problem as the first Presenter in the Fight, one as the second Presenter, and the third problem as the third Presenter. This corresponds to roles A, B, and C from Section~\ref{sec:details}. We now prove that order fairness is not a stronger criterion than feasibility.

\begin{thm}
Each feasible schedule can be transformed into an order fair schedule in polynomial time. 
\end{thm}
\begin{proof}
A feasible schedule is given by the assignments $\mathcal{P}$, $\mathcal{R}$, and~$\mathcal{O}$. Our task is, based on the pair $\mathcal{P}$ and $\mathcal{R}$, to construct the allocation $\mathcal{O}'$, which encodes the order of teams within Fights in such a way that it fulfills Definition~\ref{def_order_fair}.

This time, we reach this goal with the help of a different bipartite graph than in Theorem~\ref{t_simple}. We start with constructing the bipartite graph $H(\mathcal{P},\mathcal{R})=(T \cup F,E)$ where the sets $T$ and $F$ of vertices correspond to the set of teams and to the set of Fights---i.e., pairs $(j,r)$ where $j$ is a round and $r$ is a room---in the feasible schedule, respectively. The pair $\{t,f\}$ where $f=(j,r)$ is an edge in $H$ if and only if $\varphi_j(t)=r$, i.e., team $t$ is assigned in round $j$ to room~$r$. 

An ordering of teams in Fights corresponds to an edge coloring in $H$ by four colors A, B, C, and D, with a special condition: color D can only be used for edges incident to vertices in $F$ that are of degree~4, i.e., based on rooms from~$R_4$. Team $t$ plays the role of the first Presenter in Fight $f$ if edge $\{t,f\}$ is colored by A. Similar holds for the remaining three colors. The special condition on color D is necessary, because the role of a fourth Presenter should only be allocated to 4-Fights. 

We propose a simple algorithm to construct an edge coloring respecting our conditions. In the first step, we calculate a matching $M_D$ covering all vertices $f \in F$ with $\deg(f) = 4$. Such a matching is guaranteed to exist, because any vertex set of 4-Fights fulfills the Hall-criterion~\citep{Hal35}. We know that $k$ 4-Fights are adjacent to $4k$ edges, which lead to some team vertices forming the neighborhood of the $k$ 4-Fights. Each of these team vertices is counted at most 3 times in the enumeration of the $4k$ edges, because of $\deg(t)=3$ in~$H$. Thus the neighborhood of the $k$ chosen vertices in $F$ has cardinality at least $k$ and so a matching $M_D$ covering all 4-Fight vertices must exist. For the edges in $M_D$ we fix color D, and remove these edges from the edge set~$E$. Notice that the maximum degree in the remainder of $H$ is~3, and each $f \in F$ now has $\deg(f) = 3$. By K\"onig's theorem, an edge coloring with 3 colors exists in this graph, and it can be found efficiently, by iteratively coloring all edges of a matching covering all vertices in $F$ with a fixed color \citep{Konig}. This coloring defines the roles A, B, and C so that each Fight will have exactly one team in each of these three roles.

This algorithm computes a maximum matching for each of the four roles. Computing such a matching is of computational complexity $O(\sqrt{|T \cup F|}|A|)$  
\citep{HK73}. Since the graph is of bounded degree, there are at most as many Fights as teams, and there is a constant number of matchings to be calculated, the computational complexity reduces to 
$O(n^{1.5})$.
\end{proof}

We now demonstrate our algorithm on the Example from Table~\ref{t_fair_schedule}, which contains a fair, but not order fair schedule for the real data from the tournament Bratislava 2018. Figure~\ref{f_B18f} depicts the bipartite graph $H(\mathcal{P},\mathcal{R})$ built for this schedule, and Table~\ref{t_order_fair_schedule} contains the schedule computed with the help of this graph.

\begin{figure}[htb]
    \centering
    \resizebox{\textwidth}{!}{
	\begin{tikzpicture}[transform shape]
	
	    \tikzstyle{edge} = [very thick]
        \tikzstyle{matchingedge} = [ultra thick, MyPurple, dashed]
        \tikzstyle{prefedge} = [ultra thick]
        \tikzstyle{prefedgeincom} = [dotted]
        \tikzstyle{scaling}=[font={\LARGE\bfseries},scale=0.5]
        \tikzstyle{edgereallabel}=[scaling]
        \tikzstyle{edgelabel} = [circle, fill=white, scaling]
	
		\pgfmathsetmacro{\b}{1.3}
		\pgfmathsetmacro{\d}{\b*3}
				
		\node[vertex,label=above:S1] (s1) at ({\b * (1 + 2)}, \d) {};
		\node[vertex,label=above:S2] (s2) at ({\b * (2 + 2)}, \d) {};
		\node[vertex,label=above:S3] (s3) at ({\b * (3 + 2)}, \d) {};
		\node[vertex,label=above:W1] (w1) at ({\b * (4 + 2)}, \d) {};
		\node[vertex,label=above:W2] (w2) at ({\b * (5 + 2)}, \d) {};
		\node[vertex,label=above:W3] (w3) at ({\b * (6 + 2)}, \d) {};
		\node[vertex,label=above:T1] (t1) at ({\b * (7 + 2)}, \d) {};
		\node[vertex,label=above:T2] (t2) at ({\b * (8 + 2)}, \d) {};
		\node[vertex,label=above:B1] (b1) at ({\b * (9 + 2)}, \d) {};
		\node[vertex,label=above:B2] (b2) at ({\b * (10 + 2)}, \d) {};
		\node[vertex,label=above:E] (e) at ({\b * (11 + 2)}, \d) {};
		\node[vertex,label=above:L] (l) at ({\b * (12 + 2)}, \d) {};
		\node[vertex,label=above:D] (d) at ({\b * (13 + 2)}, \d) {};	
		
		\node[vertex,label=below:F1.1] (f11) at ({\b * (3.5)}, 0) {};
		\node[vertex,label=below:F1.2] (f12) at ({\b * (4.5)}, 0) {};
		\node[vertex,label=below:F1.3] (f13) at ({\b * (5.5)}, 0) {};
		\node[vertex,label=below:F1.4] (f14) at ({\b * (6.5)}, 0) {};
		\node[vertex,label=below:F2.1] (f21) at ({\b * (7.5)}, 0) {};
		\node[vertex,label=below:F2.2] (f22) at ({\b * (8.5)}, 0) {};
		\node[vertex,label=below:F2.3] (f23) at ({\b * (9.5)}, 0) {};
		\node[vertex,label=below:F2.4] (f24) at ({\b * (10.5)}, 0) {};
		\node[vertex,label=below:F3.1] (f31) at ({\b * (11.5)}, 0) {};
		\node[vertex,label=below:F3.2] (f32) at ({\b * (12.5)}, 0) {};
		\node[vertex,label=below:F3.3] (f33) at ({\b * (13.5)}, 0) {};
		\node[vertex,label=below:F3.4] (f34) at ({\b * (14.5)}, 0) {};
		
		\foreach \i/\j in {w3/f24,b1/f34,d/f14} {
           \draw[edge, snake it, orange] (\i) -- (\j);
        }
				
	    \foreach \i/\j in {s1/f11,s2/f12,s3/f23,w1/f21,w2/f22,w3/f34,t1/f13,t2/f14,b2/f24,e/f32,l/f31,d/f33} {
           \draw[edge, dashed, green] (\i) -- (\j);
        }
		
         \foreach \i/\j in {s1/f21,s2/f32,s3/f13,w1/f11,w2/f12,t1/f33,t2/f34,b1/f22,b2/f31,e/f14,l/f23,d/f24} {
           \draw[edge, dotted] (\i) -- (\j);
        }
        
         \foreach \i/\j in {s1/f31,s2/f22,s3/f34,w1/f32,w2/f33,w3/f13,t1/f21,t2/f23,b1/f14,b2/f11,e/f24,l/f12} {
           \draw[edge, gray] (\i) -- (\j);
        }
        
	\end{tikzpicture}
	}
    \caption{An order fair schedule computed for the tournament Bratislava 2018. The team names are abbreviated to their first letter and team number, e.g., S1 denotes Sharks1, while the Fights can be identified based on the round and room in this order, e.g., F1.4 denotes round~1, room~4. Matching $M_D$ and role D is marked by wavy orange edges, role A is marked by dashed green edges, role B is marked by dotted black edges, and finally, role C is marked by solid gray edges.} 
    \label{f_B18f}
\end{figure}
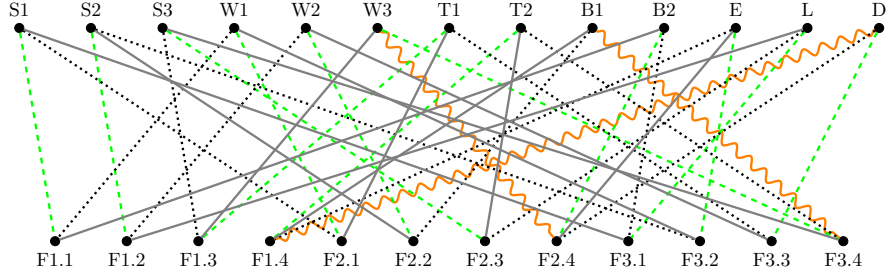

\begin{table}[ht]
\centering
\resizebox{\textwidth}{!}{
\begin{tabular}{|cl|lc|lc|lc|lc|}
   \noalign{\hrule}
 &&
\multicolumn{2}{c}{Room 1} & \multicolumn{2}{|c}{Room 2} & \multicolumn{2}{|c}{Room 3} & \multicolumn{2}{|c|}{Room 4} \\
	      \noalign{\hrule}
&&Team & Problem & Team & Problem & Team & Problem   & Team & Problem \\
   \noalign{\hrule}
\parbox[t]{2mm}{\multirow{4}{*}{\rotatebox[origin=c]{90}{{\bf Round 1}}}}
&A&Sharks1  & 6       & Sharks2 & 16    &  Turtles1  & 3   &  Turtles2 & 10  \\
&B& Whales1  & 3       & Whales2 & 12     & Sharks3 &1     &  Eagles & 9  \\
&C&Bears2  & 9    & Lions & 9   & Whales3  & 4      &  Bears1  & 8 \\
&D&&                      &&                     &&                          & Dogs & 7         \\    
 \noalign{\hrule}
&& 
\multicolumn{2}{c}{Room 1} & \multicolumn{2}{|c}{Room 2} & \multicolumn{2}{|c}{Room 3} & \multicolumn{2}{|c|}{Room 4} \\
	      \noalign{\hrule}
&&Team & Problem & Team & Problem & Team & Problem   & Team & Problem \\
   \noalign{\hrule}
\parbox[t]{2mm}{\multirow{4}{*}{\rotatebox[origin=c]{90}{{\bf Round 2}}}}
&A&Whales1  & 7        & Whales2 & 2     &  Sharks3 & 7   &  Bears2 & 5  \\
&B&Sharks1  & 4       & Bears1 & 4      &  Lions  & 4       &  Dogs  & 3  \\
&C&Turtles1  & 2        & Sharks2 & 10  &  Turtles2  & 6       &  Eagles & 16  \\
&D &  &                      &&                     &&                          & Whales3  & 10         \\    
\noalign{\hrule}
&& 
\multicolumn{2}{c}{Room 1} & \multicolumn{2}{|c}{Room 2} & \multicolumn{2}{|c}{Room 3} & \multicolumn{2}{|c|}{Room 4} \\
	      \noalign{\hrule}
&&Team & Problem & Team & Problem & Team & Problem   & Team & Problem \\
   \noalign{\hrule}
\parbox[t]{2mm}{\multirow{4}{*}{\rotatebox[origin=c]{90}{{\bf Round 3}}}}
&A&Lions  & 10                & Eagles & 4  &  Dogs  & 4        &  Whales3 & 9  \\
&B&Bears2  & 17            & Sharks2 & 17      &  Turtles1  & 14         &  Turtles2 & 5  \\
&C&Sharks1  & 14           & Whales1 & 14 & Whales2 & 5        & Sharks3  & 13    \\
 &D&&                    &&                     &&                                         &  Bears1  & 3       \\    
   \noalign{\hrule}
\end{tabular}
}
\caption{A fair and order fair schedule for the regional tournament Bratislava 2018.}\label{t_order_fair_schedule}
\end{table}

\section{Integer program for a fair schedule}\label{s_ILP}

In this section we present a family of integer linear programs to find fair schedules. First, we develop a compact representation of portfolios in order to model the scheduling problem with a small number of variables. Then we define the ILPs corresponding to our fairness notions.

We assume that the profile 
$\Pi$ is given in the form of triples, where $P(t_i)=(p^i_1,p^i_2,p^i_3)$ denotes the three problems  in the portfolio of team~$t_i$. We denote by $\ell(i,q)$ the index of the problem that is in the $q^{th}$ position in the portfolio of team~$t_i$, $i\in [n], q\in[3]$. 
Further, we construct for each $\ell\in[m]$ the list $T(\ell)$ of pairs $(i,q)$ such that problem $p_\ell$ is the $q^{th}$ problem in the portfolio of team $t_i$, i.e., 
$$T(\ell)=\{(i,q)\ |\ i\in [n]; \ p^i_q=p_\ell.\}$$
Let matrix $C$ with $n$ rows and $m$  columns be

$$c_{i\ell}  =
\left\{\begin{array}{ll}
1 & \mbox{\ if\ } p_\ell\in P(t_i)\\
0 & \mbox{\ otherwise.}
\end{array}\right.
$$

We illustrate this notation using the profile from Table~\ref{t_appl}. Let us consider team Sharks1 to be team~$t_1$. Then $P(t_1)=\{p_4,p_6,p_{14}\}$, hence $\ell(1,1)=4, \ell(1,2)=6$, and $\ell(1,3)=14$. If we take problem $p_{16}$, then $T(16)=\{(2,2),(11,3)\}$, when we take team Sharks2  to be team $t_2$ and team Eagles to be team $p_{11}$.

\noindent Let us introduce binary variables
$$x_{ijkq} \in \left\{ 0,1\right\} \mbox{ \quad for\quad} i\in [n]; \quad j\in[3]; \quad k\in [s]\quad q\in[3]$$
with the following interpretation.
$$x_{ijkq} =
\left\{\begin{array}{ll}
1 & \mbox{\ if team $t_i$ presents the $q^{th}$ problem from its portfolio in round $j$ in room $r_k$}\\
0 & \mbox{\ otherwise}
\end{array}\right.
$$
A feasible schedule is defined by the following system of equations and inequalities:
\begin{eqnarray}
\sum_{j=1}^3\sum_{k=1}^s x_{ijkq}&=& 1 \quad \mbox{\ for each team $t_i$ and each $q\in[3]$} \label{enn2}\\
\sum_{k=1}^s\sum_{q=1}^3 x_{ijkq}&=& 1 \quad \mbox{\ for each team $t_i$ and each round $j$} \label{en2}\\
\sum_{i=1}^n\sum_{q=1}^3 x_{ijkq}  &=& \size(r_k) \quad 
\mbox{\ for each round $j$ and each room $r_k\in R$ }
\label{en4}\\
\sum_{(i,q)\in T(\ell)} x_{ijkq} &\le &1 \quad 
\mbox{\ for each round $j$, each room $r_k$, and each problem $p_\ell$ } \label{en6}
\end{eqnarray} 
Binary solutions of system (\ref{enn2})-(\ref{en6}) correspond to feasible schedules, because these equations and inequalities mean the following.

\noindent (\ref{enn2}): Each team presents each problem from its portfolio exactly once.

\noindent (\ref{en2}): Each team presents in each round exactly one problem.

\noindent (\ref{en4}): 
In each round and in each room $r_k$ the number of presented problems is equal to~$\size(r_k)$.

\noindent (\ref{en6}): In each round and each room each problem is presented at most once.

These conditions are equivalent to criteria (a)--(c) from Section~\ref{sec:details}. We remark that choosing an order of Presenters so that it fulfills criterion~(d) from Section~\ref{sec:details} can be done easily, for example by assigning roles A, B, C---and D, in the case of 4-rooms---in an increasing order of team indices in each Fight. Order fairness, if required, can be obtained afterwards using our algorithm in Section~\ref{s_order}.



Let us now recall the fairness condition from Definition~\ref{def_fair}. A feasible schedule is {\bf fair} if the following holds: If team $t_\alpha$ is in some round $j$  in a room $r_k$ together with team $t_i$ who presents problem $p_\ell$ then $p_\ell\notin P(t_\alpha)$. This can be expressed by the following inequality:

\begin{equation}\label{en_fair2}
x_{ijkq}+\sum_{w=1}^3 x_{\alpha jkw}+c_{\alpha \ell(i,q)}\le 2
\mbox{\ for each $k\in [s]$, each $q\in[3]$, and each pair $i\ne \alpha$}.
\end{equation} 

Let us see how (\ref{en_fair2}) ensures fairness. Assume that team $t_i$ presents the problem that is stated in the $q^{th}$ position in $P(t_i)$ 
during the Fight that takes place in room $r_k$ in round~$j$. This means that $x_{ijkq}=1$. Team $t_\alpha$ is assigned to the same Fight if and only if it presents some problem in room $r_k$ in round $j$; this holds if and only if the second term on the left hand-side of inequality (\ref{en_fair2}) is equal to 1. In this case,  inequality (\ref{en_fair2})  implies $c_{\alpha \ell(i,q)}=0$, i.e., problem $p_{\ell(i,q)}$ is not in the portfolio of team~$t_\alpha$. This discussion implies the following assertion.

\begin{thm}
Fair schedules for IYPT correspond to the solutions of the integer linear program consisting of the feasibility constraints (\ref{enn2})--(\ref{en6}) and the fairness constraint (\ref{en_fair2}) formulated for each round~$j\in[3]$. Weakly fair schedules correspond to the solutions of (\ref{enn2})--(\ref{en6}) and the  constraint (\ref{en_fair2}) for~$j=1,2$.
\end{thm}

Let us now consider the strong fairness condition. Recall that a feasible schedule is strongly fair if no team $t_i$ deals with a problem $p_\ell$ more than once during the tournament (in either role). To formulate this condition, we introduce another set of non-negative variables:
$$y_{ijk\ell} \ \ge 0
\mbox{\quad for\quad } i\in [n]; \quad j\in[3]; \quad k\in [s];\quad \ell\in[m].$$ 
The desired interpretation of these variables is such that  $y_{ijk\ell}\ge 1$  if team $t_i$ can see problem $p_\ell$ during its presentation in round $j$ in room $r_k$; this is ensured by the inequalities (\ref{en_sfair}):
\begin{equation}\label{en_sfair}
y_{ijk\ell}\ge \sum_{w=1}^3 x_{ijkw}+\sum_{(\alpha,q)\in T(\ell)} x_{\alpha ijq}-1
\mbox{\ \ \ for each\ } i\in [n],j\in [3],k\in [s],\ell\in [m].
\end{equation}

To see this, notice that the first sum on the right-hand side is equal to 1 if team $t_i$ presents some problem in round $j$ in room $r_k$, which is equivalent to team $t_i$ being in this room in the respective round, otherwise it is equal to 0. The second sum is equal to 1 if problem $p_\ell$ is presented in round $j$ in room $r_k$ by some team $t_\alpha$, otherwise it is equal to 0. 
The inequalities ensuring strong fairness are
\begin{equation}\label{sf_2}
\sum_{j=1}^3\sum_{k=1}^s y_{ijk\ell}\le 1 \mbox{\ for each $i\in [n]$ and each $\ell\in [m]$},
\end{equation}
as they mean that each team $t_i$ can see any problem $p_\ell$ at most once.

\begin{thm}\label{thm:sfair:2}
 Strongly fair schedules for IYPT correspond to the solutions of the integer linear program consisting of the feasibility constraints (\ref{enn2})--(\ref{en6}), and inequalities (\ref{en_sfair}) and (\ref{sf_2}).
\end{thm}

Finally, we express the condition of non-cooperativity in the form of an inequality. 
A schedule is non-cooperative if the inequality 
\begin{equation}\label{sf_2_d}
\sum_{i\in T_\lambda}\sum_{q=1}^3 x_{ijkq}\le 1.
\end{equation}
 holds for each $j\in[3]$, each $k\in [s]$ and each $\lambda\in[\Lambda]$.

\section{Computations}\label{s_comp} 
We now present our computational work on real and generated data in Sections~\ref{s_real} and~\ref{s_simulation}, respectively. 

\subsection{Real data}\label{s_real}

The members the two regional committees of the IYPT in Slovakia (based in Bratislava and in Ko\v sice) provided us with the portfolios for the years 2018 and 2019. They also showed us the schedules they prepared for  regional tournaments in these years. Let us mention here that all schedules used in reality were non-cooperative, but none of them was fair. We even encountered a team that had seen presentations of two of its problems before it presented them---see team Lions in Table~\ref{t_real_schedule}. 

We attempted to compute schedules that are non-cooperative and fair. In our simulations we used the open source solver \textit{lpsolve} \citep{BDE+07}, version 5.5 under Java wrapper library. We kept the default parameter settings for integer and mixed integer problems. The solver was running on a desktop computer with the processor Intel (R) Core (TM) i5-2500 3.3 GHz and 6 GB RAM.

A summary of the computations with real data is given in Table~\ref{t_5}. The columns contain the number of teams, the number of 3-rooms and 4-rooms, the number of variables and constraints in the constructed ILP, the computation time in seconds, and the degree of fairness, respectively. In 2018, the organizers of the regional tournament in Ko\v sice used three 4-rooms and only one 3-room for 15 teams and for this case the solver was not able to find an answer concerning  fair schedule (either output a solution  or the answer that the problem is infeasible)  within 30 minutes. However, we found a non-cooperative weakly fair solution for this case, and also a non-cooperative fair solution if the 15 teams were scheduled to fill up five 3-rooms. For all other portfolios from  years 2018 and 2019 we obtained a non-cooperative fair schedule within seconds.

2020 was the first year when we were involved in the preparation of the schedules for regional tournaments. In this year, the organizers expanded the number of local rounds from 2 to 4, in order to provide access to the competition to students from rural schools. The two additional tournaments took place in smaller cities Martin (closer to Bratislava) and Poprad (closer to Ko\v sice), each of them involving just 3 teams. This drained the Bratislava round to only 5 participating teams, while the Ko\v sice event took place with 13 teams. Moreover, one of the teams in Bratislava, called MIX, involved students from different schools and the portfolio of MIX contained the same problem twice. This highly unusual makeshift team does not fit the standard input conditions. For these reasons, our only real challenge was to prepare a non-cooperative fair schedule for the Ko\v sice event, which we succeeded in. Additionally, a non-cooperative fair schedule were also possible if the three teams from  Poprad would have joined the Ko\v sice event. We merged the teams at the Bratislava and the Martin events, and deleted the makeshift team MIX. For such a  tournament, the solver found within 0.09 seconds that the ILP for a non-cooperative fair schedule was infeasible, and within 0.056 seconds that even constructing a weakly fair variant was infeasible. This was probably due to the highly correlated profiles submitted---they completely avoided 6 out of the 17 problems and 2 problems were chosen by 4 out of 8 teams.

We remark that for the strong fairness criterion, the solver did not reach any conclusion within time limit of 30 minutes for any of the instances from the years 2018-2020.

Year 2021 was again special. Because of the pandemic, only one tournament  in Slovakia was organized with 9 participating teams, moreover, it took place online. One of the teams proposed to present either problem 13 or problem 17. We offered a strongly fair solution using problem 17. When we used problem 13, we were able to compute a fair solution, however, in this solution the same composition of a Fight was repeated in all three rounds. For the strong fairness criterion we were not able to reach any conclusion within more than one hour.

\begin{table}[ht]
\begin{center}
\resizebox{\textwidth}{!}{
\begin{tabular}{|c|ccccccc|}
	      \noalign{\hrule}
 File  & teams &  3-rooms &  4-rooms &  variables
&  constraints & run-time (s) & result\\
	      \noalign{\hrule}
KE2018 & 15 & 5 & 0 & 675 & 11 220 & \phantom{0}3.43 & Fair \\
KE2018 & 15 & 1 & 3 & 540 & \phantom{0}8 994 & 30 min & TimeOut \\
KE2018 & 15 & 1 & 3 & 540 & \phantom{0}6 474  & \phantom{0}6.51 & Weakly fair \\
BA2018 & 13 & 3 & 1 & 468 & \phantom{0}6 894 & \phantom{0}6.38 & Fair \\
KE2019 & 13 & 3 & 1 & 468 & \phantom{0}6 870 &  48.86 & Fair \\
BA2019 & \phantom{0}9 & 3 & 0 & 243 & \phantom{0}2 673 & \phantom{0}0.09 & Fair \\
KE2020 & 13 & 3 & 1 & 468 & \phantom{0}6 870 &  95.36 & Fair \\
KE+PP2020 & 16 & 4 & 1 & 720 & 12 682 & \phantom{0}7.42 & Fair\\
Slovakia2021 & 9 & 3 & 0 & 1620 & 5580 & 1.26 & Strongly fair\\
      \noalign{\hrule}
    \end{tabular}
    }
\end{center}
\caption{Summary of computations of non-cooperative fair schedules for real tournaments. }\label{t_5} 
\end{table}

\subsection{Randomly generated data}\label{s_simulation}

We randomly generated profiles that resemble  situations that could typically occur in practice. The structure of the generated samples was derived from the structure of profiles in recent years and from our knowledge of the situation in Physics education and schools in the respective regions.

Teams for the competition are nominated by schools and we assume that a `big’ school nominates between 2 and 4 teams whilst a `small’ school nominates 1 or 2 teams. Higher numbers were less probable. In more detail, we set the probabilities that a big school nominates 2, 3, and 4 teams to 
0.5, 0.3, and 0.2, respectively. For small schools, the probability of nominating one team was 0.75 and that of nominating 2 teams 0.25. Further, we assumed that not all problems are equally popular. Based on the situation in 2018 and 2019 we estimated that in the set of 17 published problems there are 8 problems with low popularity, 6 problems with medium popularity and 3 problems with high popularity. We assumed that a team chooses a problem of low popularity with probability $\mu$, a problem of medium popularity with probability $2\mu$ and a problem of high popularity with probability~$4\mu$. 

We generated 50 samples for region Bratislava and another 50 samples for region Ko\v sice.  We assumed that in region Bratislava there are 3 big schools and 3 small schools, whilst in region Ko\v sice there are 2 big schools and 6 small schools. The number of teams $n$ in the generated samples was between 9 and 15 for Bratislava and it was between 10 and 16 for Ko\v sice.

The results of computations of non-cooperative weakly fair, fair, and strongly fair schedules are summarized in Table~\ref{t_7}. The column labelled {\it undecided} shows the number and ratio of instances for which the solver stopped after 5 minutes due to the prescribed time-out without any result. Computation times are summarized separately for feasible and infeasible instances. Notice that we performed the computations of fair and strongly fair schedules even for instances where we already knew that a schedule fulfilling a weaker form of fairness does not exist so as to obtain a comparison of computation times. 

The computations depicted in 
Table ~\ref{t_7}
correspond to the choice of room sizes that follow the international rules. This means that 4-rooms are only used when necessary, i.e., when the number of teams $n$ is not divisible by 3, hence the number of 4-rooms is 0, 1, or 2. However, sometimes the organizers of regional tournaments want to minimize the number of rooms used and prefer 4-rooms. A different composition of room sizes is possible in our case if $n=12$, 15 or 16. The number of instances with such $n$ among Bratislava-type data was 14 and among Ko\v sice-type data it was 21. Notice that for $n=12$ and $n=16$, a schedule that uses only 4-rooms is possible, and for $n=15$, one can use 3 rooms of size 4 and one 3-room. In this case chances of the existence of a fair schedule are much lower. For Bratislava region and non-cooperative weak fairness, 8 instances out of 14 were infeasible, for 3 of them the solver was not able to find an answer within 1 hour, and only 3 instances admitted a weakly fair schedule; for one of them the answer was output after 19 minutes.  These results  are given in Table~\ref{t_10}.

\begin{table}[ht]
\begin{center}
\resizebox{\textwidth}{!}{
\begin{tabular}{|l|ccc|cc|cc|}
	      \noalign{\hrule}
   &  \multicolumn{3}{c|}{Number and ratio of instances}            
   & \multicolumn{2}{c|}{CPU  time (feasible)} 
   & \multicolumn{2}{c|}{CPU  time (infeasible)}\\
 Criterion     & infeasible   & undecided & feasible &  median & maximum &  median & maximum \\
	      \noalign{\hrule}
	      Bratislava\\
	      \noalign{\hrule}
Weakly fair  & 6 (12\%) & 7 (12\%) & 37 (74\%)    & 0.29  & 156.51  & 8.86 & 231.49\\
Fair    & 7 (14\%)  & 14 (24\%)  & 29 (58\%)     & 0.61  & 112.64 & 2.65 & 239.93 \\
Strongly fair & 6  (12\%)  & 43 (86\%)           & 1 (2\%)       & 1.22  & \phantom{00}1.22 & 1.26 & 162.46  \\
      \noalign{\hrule}
Ko\v sice\\    
    	      \noalign{\hrule}
Weakly fair  & 2 (4\%) & 3 (6\%)  & 45 (90\%)    & \phantom{0}0.66  & \phantom{0}47.48  & 0.53 & \phantom{0}91.56\\
Fair    & 2 (4\%)   & 20 (40\%)  & 28 (56\%)    & \phantom{0}2.27    & 269.63   & 1.42 & 169.61\\
Strongly fair & 2  (4\%) & 47 (94\%)  & 1 (2\%)      & 93.84  & \phantom{0}93.84  & 0.72 & \phantom{0}23.04\\
      \noalign{\hrule}
    \end{tabular}
    }
\end{center}
\caption{Summary  of computations for randomly generated data - room sizes according to the international rules.}\label{t_7}
\end{table}

\begin{table}[ht]
\begin{center}
\resizebox{\textwidth}{!}{
\begin{tabular}{|l|ccc|cc|cc|}
	      \noalign{\hrule}
   &  \multicolumn{3}{c|}{Number and ratio of instances}            
   & \multicolumn{2}{c|}{CPU  time (feasible)} 
   & \multicolumn{2}{c|}{CPU  time (infeasible)}\\
 Criterion     & infeasible   & undecided & feasible &  median & maximum  &  median & maximum \\
	      \noalign{\hrule}
Ko\v sice\\
	      \noalign{\hrule}
Weakly fair  & 7   (33\%)               & 6 (29\%)          & 8 (38\%)    & 0.94 & 189.04  & \phantom{0}4.14 & 399.45 \\
Fair & 7     (33\%)  & 10   (48\%)   & 4  (19\%)  & 3.68   & \phantom{0}26.06  & 12.60 & 572.56 \\
Strongly fair & 7  (33\%)   & 14 (67\%)               & 0      & n.a.  & n.a.  & \phantom{0}1.20 & \phantom{0}39.82 \\
      \noalign{\hrule}
        Bratislava\\
	      \noalign{\hrule}
Weakly fair  & 8   (57\%)   & 3 (21\%)          & 3 (21\%)    & 1.1 & 100.92  & \phantom{0}0.1 & 21.62 \\
Fair & 2     (66\%)  & 1   (33\%)   & 0    & n.a.   & n.a.  & 1.1 & 1166.8 \\
      \noalign{\hrule}
    \end{tabular}
    }
\end{center}
\caption{Summary of computations with minimum number of rooms, randomly generated data.}\label{t_10}
\end{table}

\section{Conclusion}

In this paper we studied the scheduling problem arising in the organization of regional competitions of the International Young Physicist Tournament. Based on considerations of organizers, we introduced novel fairness criteria for scheduling problems. To find fair schedules we proposed integer linear programs, applied them successfully to real profiles from recent years, and explored their behaviour on randomly generated data.

Our simulations revealed that if teams are allowed to choose their portfolios completely arbitrarily, then the chances of a non-cooperative fair schedule may be low. Let us therefore think about another approach. Suppose that instead of submitting a fixed portfolio, each team submits a preference ordering of the problems---perhaps it might even be allowed to label some problems as unacceptable. We seek a matching of teams to triples of problems, which enables a fair schedule, and is in a sense optimal. Several optimality criteria can be thought of, for example minimizing the position of the least preferred problem in the final portfolio of each team, or minimizing the weighted sum of ranks of assigned problems in the portfolio.

Notice that 
we leave the theoretical complexity of the existence of a fair schedule open. The feasibility constraints
(\ref{enn2})--(\ref{en4}) resemble a {\it multi-index transportation problem (MITP)} \citep{QS01} known to be an NP-hard problem, but constraint (\ref{en6}) and the additional constraints on non-cooperativity and fairness make our problem different.

Practically, in some cases it is easy to see why a fair schedule does not exist, e.g., if the portfolios are too similar to each other. The next theoretical step could be deriving some easily verifiable combinatorial certificate for unsolvable fair schedule instances.

We hope to have opened a new perspective on scheduling student competitions with our work. Our ILP model seems to be useful for the preparation of fair schedules of regional tournaments that are consistent with the IYPT rules of at least four countries: Austria, Germany, Slovakia, and Switzerland. As we have already stated in Section~\ref{sec:intro}, German organizers formalize conditions to be fulfilled in decreasing order of importance. Their first condition is our non-cooperativity constraint, then they restrict repeating roles in a slightly different manner as in our order fairness constraint, and finally, they state that each team should meet 6 different \textit{teams} in the course of the tournament, which is similar to our strong fairness, which enforces a lower bound on the number of \textit{problems} a team encounters during the tournament.

Furthermore, other competition schedules could potentially be automatized as well. A good starting point here is the analogous version of IYPT in mathematics, the International Tournament of Young Mathematicians. By applying an ILP approach to the rules at The World Universities Debating Championship or other debating tournaments we could also potentially determine fair schedules for debate rooms.

\bibliography{fyzboj}
\end{document}